\def\year{2019}\relax
\documentclass[letterpaper]{article} 
\usepackage{aaai19}  
\usepackage{times}  
\usepackage{helvet}  
\usepackage{courier}  
\usepackage{url}  
\usepackage{graphicx}  
\usepackage[utf8]{inputenc} 
\usepackage[T1]{fontenc}    
\usepackage{booktabs}       
\usepackage{amsfonts}       
\usepackage{nicefrac}       
\usepackage{microtype}      
\usepackage{comment}
\usepackage{color}

\usepackage{amsmath,amssymb,amsfonts}
\usepackage{xpatch,booktabs,xcolor,amsmath}
\usepackage{enumerate}
\usepackage{mathrsfs}
\usepackage{multirow}
\usepackage{tabularx,ragged2e,booktabs}
\usepackage{amsthm}

\theoremstyle{definition}
\newtheorem{definition}{Definition}[section]
\newtheorem{example}{Example}[section]
\newtheorem{remark}{Remark}[section]
\newtheorem{proposition}{Proposition}[section]

\newcommand{\field}[1]{\ensuremath{\mathbb{#1}}}
\newcommand{\sets}[1]{\ensuremath{\mathcal{#1}}}

\newcommand{\reals}{\ensuremath{\field{R}}} 
\newcommand{\I}[1]{\ensuremath{\mathcal{I}\left(#1\right)}} 
\newcommand{\PR}{\ensuremath{\mathsf{P}}} 
\newcommand{\E}{\ensuremath{\mathsf{E}}} 

\newcommand{\uprx}{{\bf x}_{\overline{\rm p}}}
\newcommand{\prx}{{\bf x}_{\rm p}}
\newcommand{\uprxi}{{\bf{x}}_{\overline{\rm{p}},i}}
\newcommand{\prxi}{{\bf{x}}_{{\rm{p}},i}}
\newcommand{\uprxj}{{\bf{x}}_{\overline{\rm{p}},j}}

\newcommand{\prX}{{\mathcal X}_{\rm p}}
\newcommand{\uprX}{{\mathcal X}_{\overline{\rm p}}}

\DeclareMathOperator*{\argmin}{\mathrm{argmin}}

\usepackage{pdfpages}

\begin{document}
 
%
\title{Learning Optimal and Fair Decision Trees for \\ Non-Discriminative Decision-Making}
\author{Sina Aghaei, Mohammad Javad Azizi, Phebe Vayanos\\
CAIS Center for Artificial Intelligence in Society\\
  University of Southern California, Los Angeles, CA 90007 \\
  \{saghaei,azizim,phebe.vayanos\}@usc.edu
}
\maketitle
\begin{abstract}
In recent years, automated data-driven decision-making systems have enjoyed a tremendous success in a variety of fields (e.g., to make product recommendations, or to guide the production of entertainment). More recently, these algorithms are increasingly being used to assist socially sensitive decision-making (e.g., to decide who to admit into a degree program or to prioritize individuals for public housing). Yet, these automated tools may result in discriminative decision-making in the sense that they may treat individuals unfairly or unequally based on membership to a category or a minority, resulting in disparate treatment or disparate impact and violating both moral and ethical standards. This may happen when the training dataset is itself biased (e.g., if individuals belonging to a particular group have historically been discriminated upon). However, it may also happen when the training dataset is unbiased, if the errors made by the system affect individuals belonging to a category or minority differently (e.g., if misclassification rates for Blacks are higher than for Whites). In this paper, we unify the definitions of unfairness across classification and regression. We propose a versatile mixed-integer optimization framework for learning optimal and fair decision trees and variants thereof to prevent disparate treatment and/or disparate impact as appropriate. This translates to a flexible schema for designing fair and interpretable policies suitable for socially sensitive decision-making. We conduct extensive computational studies that show that our framework improves the state-of-the-art in the field (which typically relies on heuristics) to yield non-discriminative decisions at lower cost to overall accuracy.
\end{abstract}

\section{Introduction}
\label{sec:introduction}

Discrimination refers to the unfair, unequal, or prejudicial treatment of an individual or group based on certain characteristics, often referred to as \emph{protected} or \emph{sensitive}, including age, disability, ethnicity, gender, marital status, national origin, race, religion, and sexual orientation. Most philosophical, political, and legal discussions around discrimination assume that discrimination is morally and ethically wrong and thus undesirable in our societies~\cite{sep-discrimination}. 

Broadly speaking, one can distinguish between two types of discrimination: \emph{disparate treatment} (aka \emph{direct discrimination}) and \emph{disparate impact} (aka \emph{indirect discrimination}). Disparate treatment consists of rules \emph{explicitly} imposing different treatment to individuals that are similarly situated and that only differ in their protected characteristics. Disparate impact on the other hand does not explicitly use sensitive attributes to decide treatment but \emph{implicitly}  results in systematic different handling of individuals from protected groups.

In recent years, machine learning (ML) techniques, in particular supervised learning approaches such as classification and regression, routinely assist or even replace human decision-making. For example, they have been used to make product recommendations~\cite{Finley2016} and to guide the production of entertainment content~\cite{Kumar2018netflix}. More recently, such algorithms are increasingly being used to also assist socially sensitive decision-making. For example, they can help inform the decision to give access to credit, benefits, or public services~\cite{Byrnes2016artificialintollerance}, they can help support criminal sentencing decisions~\cite{Rudin2013policing}, and assist screening decisions for jobs/college admissions~\cite{Miller2015algorithmhire}.

Yet, these automated data-driven tools may result in discriminative decision-making, causing \emph{disparate treatment} and/or \emph{disparate impact} and violating moral and ethical standards. First, this may happen when the training dataset is biased so that the ``ground truth'' is not available. Consider for example the case of a dataset wherein individuals belonging to a particular group have historically been discriminated upon (e.g., the dataset of a company in which female employees are never promoted although they perform equally well to their male counterparts who are, on the contrary, advancing their careers; in this case, the true merit of female employees --the ground truth-- is not observable). Then, the machine learning algorithm will likely uncover this bias (effectively encoding endemic prejudices) and yield discriminative decisions (e.g., recommend male hires). Second, machine learning algorithms may yield discriminative decisions even when the training dataset is unbiased (i.e., even if the ``ground truth'' is available). This is the case if the errors made by the system affect individuals belonging to a category or minority differently. Consider for example a classification algorithm for breast cancer detection that has far higher false negative rates for Blacks than for Whites (i.e., it fails to detect breast cancer more often for Blacks than for Whites). If used for decision-making, this algorithm would wrongfully recommend no treatment for more Blacks than Whites, resulting in racial unfairness. In the literature, there have been a lot of reports of algorithms resulting in unfair treatment, e.g., in racial profiling and redlining \cite{squires2003insuranceprofiling}, mortgage discrimination \cite{LaCOUR-LITTLE1999}, personnel selection \cite{Stoll2004blackjobapplicants}, and employment \cite{Kuhn1990sexdiscriminationlabor}. Note that a ``naive'' approach that rids the dataset from sensitive attributes does not necessarily result in fairness since unprotected attributes may be correlated with protected ones.

In this paper, we are motivated from the problem of using ML for decision- or policy-making in settings that are socially sensitive (e.g., education, employment, housing) given a labeled training dataset containing one (or more) protected attribute(s). The main desiderata for such a data-driven decision-support tool are: (1) Maximize predictive accuracy: this will ensure that e.g., scarce resources (e.g., jobs, houses, loans) are allocated as efficiently as possible, that innocent (guilty) individuals are not wrongfully incarcerated (released); (2) Ensure fairness: in socially sensitive settings, it is desirable for decision-support tools to abide by ethical and moral standards to guarantee absence of disparate treatment and/or impact; (3) Applicable to both classification and regression tasks: indeed, disparate treatment and disparate impact may occur whether the quantity used to drive decision-making is categorical and unordered or continuous/discrete and ordered; (4) Applicable to both biased and unbiased datasets: since unfairness may get encoded in machine learning algorithms whether the ground truth is or not available, our tool must be able to enforce fairness in either setting; (5) Customize interpretability: in socially sensitive settings, decision-makers can often decide to comply or not with the recommendations of the automated decision-support tool; recommendations made by interpretable systems are more likely to be adhered to; moreover,  since interpretability is subjective, it is desirable that the decision-maker be able to customize the structure of the model. Next, we summarize the state-of-the-art in related work and highlight the need for a unifying framework that addresses these desiderata.

\subsection{Related Work}

\textbf{Fairness in Machine Learning.} The first line of research in this domain focuses on \emph{identifying} discrimination in the data~\cite{Pedreshi:2008} or in the model \cite{Adler:2018:ABM:3182400.3182507}. The second stream of research focuses on \emph{preventing} discrimination and can be divided into three parts. First, pre-processing approaches, which rely on modifying the data to eliminate or neutralize any preexisting bias and subsequently apply standard ML techniques~\cite{Kamiran2012,Kamiran2009classnodiscr,Luong:2011}. We emphasize that preprocessing approaches cannot be employed to eliminate bias arising from the algorithm itself. Second, post-processing approaches, which a-posteriori adjust the predictors learned using standard ML techniques to improve their fairness properties~\cite{Hardt:2016,Fish:2016}. The third type of approach, which most closely relates to our work, is an in-processing one. It consists in adding a fairness regularizer to the loss function objective, which serves to penalize discrimination, mitigating disparate treatment~\cite{Dwork:2012:FTA:2090236.2090255,pmlr-v28-zemel13,Berk:2017} or disparate impact~\cite{Calders:2010,Kamiran:2010}. Our approach most closely relates to the work in~\cite{Kamiran:2010}, where the authors propose a \emph{heuristic} algorithm for learning fair decision-trees for \emph{classification}. They use the non-discrimination constraint to design a new splitting criterion and pruning strategy. In our work, we propose in contrast an \emph{exact} approach for designing very general classes of fair decision-trees that is applicable to both \emph{classification and regression} tasks. \\ 
\textbf{Mixed-Integer Optimization for Machine Learning.} Our paper also relates to a nascent stream of research that leverages mixed-integer programming (MIP) to address ML tasks for which heuristics were traditionally employed \cite{Bertsimas:2015,Lou:2013,Mazumder:2015,Bertsimas2017,verwer2017}. Our work most closely relates to the work in~\cite{Bertsimas2017} which designs optimal \emph{classification} trees using MIP, yielding average absolute improvements in out-of-sample accuracy over the state-of-art CART algorithm~\cite{breiman1984classification} in the range 1–5\%. It also closely relates to the work in~\cite{verwer2017} which introduces optimal \emph{decision trees} and showcases how 
discrimination aware decision trees can be designed using MIP. Lastly, our framework relates to the approach in~\cite{Azizi2018_CPAIOR} where an MIP is proposed to design dynamic decision-tree-based resource allocation policies. Our approach moves a significant step ahead of~\cite{Bertsimas2017}, \cite{verwer2017}, and~\cite{Azizi2018_CPAIOR} in that we introduce a unifying framework for designing fair decision trees and showcase how different fairness metrics (quantifying disparate treatment and disparate impact) can be explicitly incorporated in an MIP model to support fair and interpretable decision-making that relies on either categorical or continuous/ordered variables. Our approach thus enables the generalization of these MIP based models to {general decision-making tasks} in \emph{socially sensitive settings} with \emph{diverse fairness requirements}. Compared to the regression trees introduced in \cite{verwer2017}, we consider more flexible decision tree models which allow for linear scoring rules to be used at each branch and at each leaf -- we term these ``linear branching'' and ``linear leafing'' rules in the spirit of \cite{Azizi2018_CPAIOR}. Compared to \cite{Bertsimas2017,verwer2017} which require one hot encoding of categorical features, we treat branching on categorical features explicitly yielding a more interpretable and flexible tree. \\ 
\textbf{Interpretable Machine Learning.} Finally our work relates to interpretable ML, including works on decision rules \cite{wangBayesian2017,letham2015interpretable}, decision sets \cite{lakkaraju2016interpretable}, and generalized additive models \cite{lou2013accurate}. In this paper, we build on decision trees~\cite{breiman1984classification} which have been used to generate interpretable models in many settings \cite{valdes2016mediboost,huang2007iptree,che2016interpretable}. Compared to this literature, we introduce two new model classes which generalize decision trees to allow more flexible branching structures (linear branching rules) and the use of a linear scoring rule at each leaf of the tree (linear leafing). An approximate algorithm for designing classification trees with linear leafing rules was originally proposed in~\cite{Frank1998}. In contrast, we propose to use linear leafing for regression trees. Our approach is thus capable of integrating linear branching and linear leafing rules in the design of fair regression trees. It can also integrate linear branching in the design of fair classification trees. Compared to the literature on interpretable ML, we use these models to yield general interpretable and fair automated decision- or policy-making systems rather than learning systems. By leveraging MIP technology, our approach can impose very general interpretability requirements on the structure of the decision-tree and associated decision-support system (e.g., limited number of times that a feature is branched on). This flexibility make it particularly well suited for socially sensitive settings.

\subsection{Proposed Approach and Contributions}

Our main contributions are:
\begin{enumerate}[(1)]
\item We formalize the two types of discrimination (disparate treatment and disparate impact) mathematically for both classification and regression tasks. We define associated indices that enable us to quantify disparate treatment and disparate impact in classification and regression datasets.
\item We propose a unifying MIP framework for designing \emph{optimal} and \emph{fair} decision-trees for \emph{classification} and \emph{regression}. The trade-off between accuracy and fairness is conveniently tuned by a single, user selected parameter. 
\item Our approach is the first in the literature capable of designing fair regression trees able to mitigate both types of discrimination (disparate impact and/or disparate treatment) thus making significant contributions to the literature on fair machine learning. 

\item  Our approach also contributes to the literature on (general) machine learning since it generalizes the decision-tree-based approaches for classification and regression (e.g., CART) to more general branching and leafing rules incorporating also interpretability constraints.

\item  Our framework leverages MIP technology to allow the decision-maker to conveniently tune the interpretability of the decision-tree by selecting: the structure of the tree (e.g., depth), the type of branching rule (e.g., score based branching or single feature), the type of model at each leaf (e.g., linear or constant). This translates to customizable and interpretable decision-support systems that are particularly attractive in socially sensitive settings.
\item We conduct extensive computational studies showing that our framework improves the state-of-the-art to yield non-discriminating decisions at lower cost to overall accuracy.
\end{enumerate}

\section{A Unifying Framework for Fairness in Classification and Regression}
\label{sec:fairness}

In supervised learning, the goal is to learn a mapping $f_{\bf \theta}:\mathbb R^d \rightarrow \mathbb R$, parameterized by ${\bf \theta} \in \Theta \subset \mathbb R^n$, that maps feature vectors ${\bf x} \in \mathcal X \subseteq \mathbb R^d$ to labels $y \in \mathcal Y \subseteq \mathbb R$. We let $\PR$ denote the joint distribution over $\mathcal X \times \mathcal Y$ and let $\E(\cdot)$ the expectation operator relative to $\PR$. If labels are \emph{categorical and unordered} and $|\mathcal Y| < \infty $, we refer to the task as a \emph{classification} task. In two-class (binary) classification for example, we have $\mathcal Y:=\{-1,1\}$. On the other hand if labels are \emph{continuous or ordered discrete values} (typically normalized so that $\mathcal Y \subseteq [-1,1]$), then the task is a \emph{regression} task. \color{black} Learning tasks are typically achieved by utilizing a training set $\mathcal T:=\{ {\bf x}_i, y_i \}_{i\in \mathcal N}$ consisting of historical realizations of ${\bf x}$ and $y$. The parameters of the classifier are then estimated as those that minimize a certain loss function $L$ over the training set $\mathcal T$, i.e., $\theta^\star \in \argmin_{\theta \in \Theta} L({\bf \theta},\mathcal T)$.

In supervised learning for decision-making, the learned mapping $f_{\theta^\star}$ is used to guide human decision-making, e.g., to help decide whether an individual with feature vector ${\bf x}$ should be granted bail (the answer being ``yes'' if the model predicts he will not commit a crime). In socially sensitive supervised learning, it is assumed that some of the elements of the feature vector ${\bf x}$ are sensitive. We denote the subvector of ${\bf x}$ that collects all protected (resp.\ unprotected) attributes by $\prx$ with support $\prX$ (resp.\ $\uprx$ with support $\uprX$). In addition to the standard classification task, the goal here is for the resulting mapping to be non-discriminative in the sense that it should not result in disparate treatment and/or disparate impact relative to some (or all) of the protected features. In what follows, we formalize mathematically the notions of unfairness and propose associated indices that serve to measure and also prevent (see Section~\ref{sec:MIP-formulation}) discrimination.


\subsection{Disparate Impact}

Disparate impact does not explicitly use sensitive attributes to decide treatment but \emph{implicitly} results in \emph{systematic} different handling of individuals from protected groups. Next, we introduce the mathematical definition of disparate impact in classification, also discussed in~\cite{bilalzafar2017fairness,barocas2016big}.

\begin{definition}[Disparate Impact in Classification] Consider a classifier that maps feature vectors ${\bf x} \in \mathbb R^d$, with associated protected part $\prx \in \prX$, to labels $y \in \mathcal Y$. We will say that the decision-making process does not suffer from disparate impact if the probability that it outputs a specific value $y$ does not change after observing the protected feature(s) $\prx$, i.e.,
\begin{equation}
\PR( y |  \prx ) = \PR( y ) \quad \text{ for all } y \in \mathcal Y \text { and } \prx \in \prX.
\label{eq:disparate_impact_classification}
\end{equation}
\label{def:disparate_impact_classification}
\end{definition}
The following metric enables us to quantify disparate impact in a dataset with categorical or unordered labels.
\begin{definition}[DIDI in Classification] Given a classification dataset $\mathcal D:= \{ {\bf x}_i, y_i \}_{i\in \mathcal N}$, we define its Disparate Impact Discrimination Index by
\begin{equation*}
\begin{array}{ccl}
{\mathsf{DIDI}}_{\rm c}({\mathcal D}) & = & \displaystyle \sum_{y \in \mathcal Y} \sum_{\prx \in \prX} \left| 
\frac{ \left| \{ i \in \mathcal N : y_i=y   \} \right| }{ | \mathcal N | } 
\right. \\ 
&& \left. \;\; -
\displaystyle \frac{ \left| \{ i \in \mathcal N : y_i=y \; \cap \; \prxi = \prx   \} \right| }{ \left| \{ i \in \mathcal N : \prxi = \prx \right| } 
\right|.
\end{array}
\end{equation*}

The higher ${\mathsf{DIDI}}_{\rm c}({\mathcal D})$, the more the dataset suffers from disparate impact. If ${\mathsf{DIDI}}_{\rm c}({\mathcal D}) = 0$, we will say that the dataset does not suffer from disparate impact.
\label{def:DIDIc}
\end{definition}

The following proposition shows that if a dataset is unbiased, then it is sufficient for the ML to be unbiased in its errors to yield an unbiased decision-support system.
\begin{proposition}
Consider an (unknown) class-based decision process (a classifier) that maps feature vectors ${\bf x}$ to class labels $y\in \mathcal Y$ and suppose this classifier does not suffer from disparate impact, i.e., $\PR(  y |  \prx ) = \PR( y )$ for all $y \in \mathcal Y$ and $\prx \in \prX$.
Consider learning (estimating) this classifier using a classifier whose output $\tilde y \in \mathcal Y$ is such that the probability of misclassifying a certain value $y$ as $\tilde y$ does not change after observing the protected feature(s), i.e.,
\begin{equation}
\PR( \tilde y |  y, \prx ) = \PR( \tilde y |  y) \text{ for all } \tilde y, \; y \in \mathcal Y \text { and } \prx \in \prX.
\label{eq:disparate_mistreatment}
\end{equation}
Then, the learned classifier will not suffer from disparate impact, i.e., 
$
\PR( \tilde y |  \prx ) = \PR( \tilde y )$ for all $\tilde y,\prx.
$
\label{prop:relation_to_disparate_mistreatment}
\end{proposition}

\begin{proof}
Fix any $\tilde y \in \mathcal Y$ and $\prx \in \prX$. We have
\begin{align*}
\PR( \tilde y )  &= \textstyle \sum_{y \in \mathcal Y} \PR ( \tilde y | y ) \PR(y)
=  \sum_{y \in \mathcal Y} \PR ( \tilde y | y , \prx ) \PR(y | \prx )\\
&= \textstyle \sum_{y \in \mathcal Y} \PR ( \tilde y \cap y | \prx )
 = \PR( \tilde y |  \prx ).
\end{align*}
Since the choice of $\tilde y \in \mathcal Y$ and $\prx \in \prX$ was arbitrary, the claim follows.
\end{proof}
\begin{remark}
Proposition~\eqref{eq:disparate_impact_classification} implies that if we have a (large i.i.d.) classification dataset $\{ {\bf x}_i, y_i \}_{i\in \mathcal N}$ that does not suffer from disparate impact (see Definition~\ref{def:DIDIc}) and we use it to learn a mapping that maps ${\bf x}$ to $y$ and that has the property that the probability of misclassifying a certain value $y$ as $\hat y$ does not change after observing the protected feature(s) $\prx$, then the resulting classifier will not suffer from disparate impact. Classifiers with the Property~\eqref{eq:disparate_mistreatment} are sometimes said to not suffer from~\emph{disparate mistreatment}, see e.g., \cite{Bilal:2016}. We emphasize that only imposing~\eqref{eq:disparate_mistreatment} on a classifier may result in a decision-support system that is plagued by disparate impact if the dataset is discriminative.
\end{remark}

Next, we propose a mathematical definition of disparate impact in regression.

\begin{definition}[Disparate Impact in Regression] Consider a predictor that maps feature vectors ${\bf x} \in \mathbb R^d$, with associated protected part $\prx \in \prX$, to values $y \in \mathcal Y$. We will say that the predictor does not suffer from disparate impact if the expected value $y$ do not change after observing the protected feature(s) $\prx$, i.e.,
\begin{equation}
\E( y |  \prx ) = \E( y ) \quad \text{ for all } \prx \in \prX.
\label{eq:disparate_impact_regression}
\end{equation}
\label{def:disparate_impact_regression}
\end{definition}

\begin{remark}
Strictly speaking, Definition~\ref{def:disparate_impact_regression} should exactly parallel Definition~\ref{def:disparate_impact_classification}, i.e., the entire distributions should be equal rather than merely their expectations. However, requiring continuous distributions to be equal would yield computationally intractable models, which motivates us to require fairness in the first moment of the distribution only. 
\end{remark}

\begin{proposition}
Consider an (unknown) decision process that maps feature vectors ${\bf x}$ to values $y\in \mathcal Y$ and suppose this process does not suffer from disparate impact, i.e.,
$
\E(  y |  \prx ) = \E( y )$ for all $\prx \in \prX$.
Consider learning (estimating) this model using a learner whose output $\tilde y \in \mathcal Y$ is such that \begin{equation*}
\E( \tilde y - y |  \prx ) = \E( \tilde y - y ) \quad \text{ for all } \tilde y \in \mathcal Y \text { and } \prx \in \prX.
\end{equation*}
Then, the learned model will not suffer from disparate impact, i.e., 
$
\E( \tilde y |  \prx ) = \E( \tilde y )$ for all $\prx \in \prX.
$
\end{proposition}

\begin{proof}
$E(\hat y |  \prx) = \E( \hat y-y | \prx) + \E(y|\prx)$
$=E(\hat y- y) + \E(y) = \E(\hat y).$
\end{proof}

The following metric enables us to quantify disparate impact in a dataset with continuous or ordered discrete labels.
\begin{definition}[DIDI in Regression] Given a regression dataset $\mathcal D:= \{ {\bf x}_i, y_i \}_{i\in \mathcal N}$, we define its Disparate Impact Discrimination Index by
$$
{\mathsf{DIDI}}_{\rm r}({\mathcal D}) = \sum_{\prx \in \prX} \left| 
\frac{  \sum_{ i \in \mathcal N } y_i \mathcal I( \prxi = \prx) }{\sum_{ i \in \mathcal N } \mathcal I( \prxi = \prx) } 
-
\frac{1}{ | \mathcal N | }  \sum_{i \in \mathcal N} y_i
\right|
$$
where $\mathcal I (\cdot)$ evaluates to 1 (0) if its argument is true (false).
The higher ${\mathsf{DIDI}}_{\rm r}({\mathcal D})$, the more the dataset suffers from disparate impact. If ${\mathsf{DIDI}}_{\rm r}({\mathcal D}) = 0$, we will say that the dataset does not suffer from disparate impact.
\label{def:DIDIr}
\end{definition}


\subsection{Disparate Treatment}

As mentioned in Section~\ref{sec:introduction}, disparate treatment arises when a decision-making system provides different outputs for groups of people with the same (or similar) values of the non-sensitive features but different values of sensitive features. We formalize this notion mathematically.

\begin{definition}[Disparate Treatment in Classification] Consider a class based decision-making process that maps feature vectors ${\bf x} \in \mathbb R^d$ with associated protected (unprotected) parts $\prx \in \prX$ ($\uprx$) to labels $y \in \mathcal Y$. We will say that the decision-making process does not suffer from disparate treatment if the probability that it outputs a specific value $y$ given $\uprx$ does not change after observing the protected feature(s) $\prx$, i.e.,
$$\PR( y |  \uprx, \prx ) = \PR( y | \uprx ) \quad \text{ for all } y \in \mathcal Y \text { and } {\bf x} \in \mathcal X.$$
\label{def:DT_c}
\end{definition}
The following metric enables us to quantify disparate treatment in a dataset with categorical or unordered labels.
\begin{definition}[DTDI in Classification]
Given a classification dataset $\mathcal D:= \{ {\bf x}_i, y_i \}_{i\in \mathcal N}$, we define its Disparate Treatment Discrimination Index by
\begin{equation}
\begin{array}{l}
{\mathsf{DTDI}}_{\rm c}({\mathcal D})= 
\displaystyle \sum_{\begin{smallmatrix} y \in \mathcal Y ,\\ \prx \in \prX \\ j \in \mathcal N \end{smallmatrix}} 
\left| 
\frac{ \textstyle \sum_{i \in \mathcal N} d( \uprxi, \uprxj ) \mathcal I (y_i=y)  }
{ \textstyle  \sum_{i \in \mathcal N} d( \uprxi, \uprxj ) }
\right. \\ 
\left.\quad -
\frac{ \textstyle \sum_{i \in \mathcal N} d( \uprxi, \uprxj ) \mathcal I (y_i=y \cap \prxi = \prx )  }
{ \textstyle  \sum_{i \in \mathcal N} d( \uprxi, \uprxj ) \mathcal I (\prxi = \prx ) }  
\right|,
\end{array}
\label{eq:DTDIc}
\end{equation}
where $d(\uprxi, \uprxj)$ is any non-increasing function in the distance between $\uprxi$ and $\uprxj$ so that more weight is put on pairs that are close to one another. The idea of using a locally weighted average to estimate the conditional expectation is a well known technique in statistics referred to as Kernel Regression, see e.g., \cite{nadaraya1964estimating}. The higher ${\mathsf{DTDI}}_{\rm c}({\mathcal D})$, the more the dataset suffers from disparate treatment. If ${\mathsf{DTDI}}_{\rm c}({\mathcal D}) = 0$, the dataset does not suffer from disparate treatment.
\label{def:DTDIc}
\end{definition}

\begin{example}[$k$NN]
A natural choice for the weight function in~\eqref{eq:DTDIc} is
$$
d( \uprxi, \uprxj ) = \begin{cases}
1 & \text{if } \uprxi \text{ is a $k$-nearest neighbor of $\uprxj$} \\
0 & \text{else}.
\end{cases}
$$
\end{example}

Next, we propose a mathematical definition of disparate treatment in regression.
\begin{definition}[Disparate Treatment in Regression] Consider a decision-making process that maps feature vectors ${\bf x} \in \mathbb R^d$ with associated protected (unprotected) parts $\prx \in \prX$ ($\uprx$) to values $y \in \mathcal Y$. We will say that the decision-making process does not suffer from disparate treatment if
$$
\E( y |  \uprx, \prx ) = \E( y | \uprx ) \quad \text{ for all } {\bf x} \in \mathcal X.
$$
\end{definition}

The following metric enables us to quantify disparate treatment in a dataset with continuous or ordered discrete labels.
\begin{definition}[DTDI in Regression]
Given a classification dataset $\mathcal D:= \{ {\bf x}_i, y_i \}_{i\in \mathcal N}$, we define its Disparate Treatment Discrimination Index by
\begin{equation}
\begin{array}{ccl}
{\mathsf{DTDI}}_{\rm r}({\mathcal D}) = 
\displaystyle \sum_{\begin{smallmatrix} \prx \in \prX ,j \in \mathcal N \end{smallmatrix}} 
\left| 
\frac{ \textstyle  \sum_{i \in \mathcal N} d( \uprxi, \uprxj ) y_i  }
{ \textstyle  \sum_{i \in \mathcal N} d( \uprxi, \uprxj ) }
\right. \\
\left. \qquad \quad  -
\frac{ \textstyle \sum_{i \in \mathcal N} d( \uprxi, \uprxj ) \mathcal I (\prxi = \prx ) y_i  }
{ \textstyle  \sum_{i \in \mathcal N} d( \uprxi, \uprxj ) \mathcal I (\prxi = \prx ) } 
\right|,
\end{array}
\label{eq:DTDIr}
\end{equation}
where $d(\uprxi, \uprxj)$ is as in Definition~\ref{def:DTDIc}. If ${\mathsf{DTDI}}_{\rm r}({\mathcal D}) = 0$, we say that the data does not suffer from disparate treatment.
\label{def:DTDIr}
\end{definition}

\section{Mixed Integer Optimization Framework for Learning Fair Decision Trees}
\label{sec:MIP-formulation}

We propose a mixed-integer linear program (MILP)-based regularization approach for trading-off prediction quality and fairness in decision trees. 

\subsection{Overview}
\label{sec:overview}
Given a training dataset $\mathcal T:=\{ {\bf x}_i, y_i \}_{i\in \mathcal N}$, we let $\hat y_i$ denote the prediction associated with datapoint $i \in \mathcal N$ and define $\hat y:=\{\hat y_i\}_{i \in \mathcal N}$. We propose to design classification (resp.\ regression) trees that minimize a loss function $\ell_{\rm c}(\mathcal T, \hat y)$ (resp.\ $\ell_{\rm r}(\mathcal T, \hat y)$) augmented with a discrimination regularizer $\ell^{\rm d}_{\rm c}(\mathcal T,\hat y)$ (resp.\ $\ell^{\rm d}_{\rm r}(\mathcal T,\hat y)$ ). Thus, given a regulization weight $\lambda \geq 0$ that allows tuning of the fairness-accuracy trade-off, we seek to design decision trees that minimize
\begin{equation}
\ell_{\rm c/\rm r}(\mathcal T, \hat y) + \lambda \ell^{\rm d}_{\rm c/\rm r}(\mathcal T,\hat y),
\label{eq:loss_objective}
\end{equation}
where the $\rm c$ ($\rm r$) subscript refers to classification (regression).

A typical choice for the loss function in the case of classification tasks is the \emph{misclassification rate}, defined as the portion of incorrect predictions, i.e., $\ell_{\rm c}(\mathcal T, \hat y) := 1/| \mathcal N | \sum_{i \in \mathcal N} \mathcal I ( y_i \neq \hat y_i )$. In the case of regression tasks, a loss function often employed is the \emph{mean absolute error} defined as $\ell_{\rm r}({\mathcal T}, \hat y) := 1/| \mathcal N | \sum_{i \in \mathcal N} | \hat y_i - y_i |$. Both these loss functions are attractive as they give rise to linear models, see Section~\ref{sec:MILP-actual}. Accordingly, discrimination of the learned model is measured using a discrimination loss function taken to be any of the discrimination indices introduced in Section~\ref{sec:fairness}. For example, in the case of classification/regression tasks, we propose to either penalize disparate impact by defining the discrimination loss function as $\ell^{\rm d}_{\rm c/\rm r}({\mathcal T},\hat y):= \mathsf{DIDI}_{\rm c/\rm r}(\{ {\bf x}_i, \hat y_i \}_{i\in \mathcal N})$ or to penalize disparate treatment by defining the discrimination loss function as $\ell^{\rm d}_{\rm c/\rm r}({\mathcal T},\hat y):= \mathsf{DTDI}_{\rm c/\rm r}(\{ {\bf x}_i, \hat y_i \}_{i\in \mathcal N})$. As will become clear later on, discrimination loss functions combining disparate treatment and disparate impact are also acceptable. All of these give rise to linear models.

\subsection{General Classes of Decision-Trees}
\label{sec:decision-trees-general}

A decision-tree~\cite{breiman1984classification} takes the form of a tree-like structure consisting of nodes, branches, and leafs. In each internal node of the tree, a ``test'' is performed. Each branch represents the outcome of the test. Each leaf collects all points that gave the same answers to all tests. Thus, each path from root to leaf represents a classification rule that assigns each data point to a leaf. At each leaf, a prediction from the set $\mathcal Y$ is made for each data point -- in traditional decision trees, the same prediction is given to all data points that fall in the same leaf.

In this work, we propose to use integer optimization to design general classes of fair decision-trees. Thus, we introduce decision variables that decide on the branching structure of the tree and on the predictions at each leaf. We then seek optimal values of these variables to minimize the loss function~\eqref{eq:loss_objective}, see Section~\ref{sec:MILP-actual}. 

Next, we introduce various classes of decision trees that can be handled by our framework and which generalize the decision tree structures from the literature. We assume that the decision-maker has selected the depth $K$ of the tree. This assumption is in line with the literature on fair decision-trees, see~\cite{Kamiran:2010}. We let $\sets V$ and $\mathcal L$ denote the set of all branching nodes and leaf nodes in the tree, respectively. Denote by $\sets F_{\rm c}$ and $\sets F_{\rm q}$ the sets of all indices of categorical and quantitative features, respectively. Also, let $\sets F:=\sets F_{\rm c} \cup \sets F_{\rm q}$ (so that $|\mathcal F|=d$). 

We introduce the decision variables $p_{\nu j}$ which are zero if and only if the $j$th feature, $j \in \mathcal F$, is not involved in the branching rule at node $\nu \in \mathcal V$. We also let the binary decision variables $z_{il} \in \{0,1\}$ indicate if data point $i \in \mathcal N$ belongs to leaf $l \in \mathcal L$. Finally, we let $\hat y_i \in \mathcal Y$ decide on the prediction for data point $i \in \mathcal N$. We denote by $\mathcal P$ and $\hat {\mathcal Y}(z)$ the sets of all feasible values for $p$ and $\hat y$, respectively.

\begin{example}[Classical Decision-Trees] 
In classical decision-trees, the test that is performed at each internal node involves a \emph{single feature} (e.g., if the age of an individual is less than 18). Thus,
$$
\mathcal P = \left\{ p \in \{0,1\}^{| \mathcal V| \times | \mathcal X|} : \textstyle \sum_{j \in \sets F} p_{\nu j}= 1 \quad \forall \nu \in \mathcal V  \right\}
$$
and $p_{\nu j}=1$ if and only if we branch on feature $j$ at node~$\nu$. Additionally, all data points that reach the same leaf are assigned the \emph{same prediction}. Thus,
$$
\hat {\mathcal Y}(z) =  \left\{ \hat y \in \mathbb R^{| \mathcal N |} : \exists u \in \mathcal Y^{|\mathcal L|} \text{ with } \textstyle \hat y_i = \displaystyle \sum_{l \in \mathcal L} z_{il} u_l \; \forall i  \right\}.
$$
The auxiliary decision variables $u_l$ denote the prediction for leaf $l \in \mathcal L$.
\label{ex:PWC_policies}
\end{example}


\begin{example}[Decision-Trees enhanced with Linear Branching]
A generalization of the decision-trees from Example \ref{ex:PWC_policies} can be obtained by allowing the ``test'' to involve a linear function of several features. In this setting, we view all features as being quantitative (i.e., continuous or discrete and ordered) so that $\mathcal F_{\rm c}=\emptyset$ and let
$$
\mathcal P = \left\{ p \in \mathbb R^{| \mathcal V| \times | \mathcal X|} : \textstyle \sum_{j \in \sets F} p_{\nu j}= 1 \quad \forall \nu \in \mathcal V  \right\}.
$$
As before, all data points that reach the same leaf are assigned the \emph{same prediction} so that
$
\hat{\mathcal Y}(z)
$
is defined as in Example~\ref{ex:PWC_policies}.
\label{ex:PWC_policies_GP}
\end{example}


\begin{example}[Decision-Trees enhanced with Linear Leafing]
\label{ex:PWL_policies}
Another variant of the decision-trees from Example \ref{ex:PWC_policies} is one where, rather than having a common prediction for all data points that reach a leaf, a linear scoring rule is employed at each leaf. Thus,
$$
\hat {\mathcal Y}(z) = \left\{
\begin{array}{ll}
 \hat y \in \mathcal Y^{| \mathcal N |} : &
 \exists u_l \in \mathbb R^{d}, \; l\in \mathcal L \text{ with } \\
 &  \hat y_i = \textstyle \sum_{l \in \mathcal L} z_{il} u_l^\top {\bf x}_i \quad \forall i \in \mathcal N
\end{array}
\right\}.
$$
The auxiliary decision variables $u_l$ collect the coefficients of the linear rules at each leaf $l\in \mathcal L$. 
\end{example}

In addition to the examples above, one may naturally also consider decision-trees enhanced with both linear branching and linear leafing. 

We note that all sets above are MILP representable. Indeed, they involve products of binary and real-valued decision variables which can be easily linearized using standard techniques. The classes of decision trees above were originally proposed in~\cite{Azizi2018_CPAIOR} in the context of policy design for resource allocation problems. Our work generalizes them to generic decision- and policy-making tasks.

\subsection{MILP Formulation}
\label{sec:MILP-actual}

 For $\nu \in \sets V$, let $\sets L^{\rm r}(\nu)$ (resp.\ $\sets L^{\rm l}(\nu)$) denote all the leaf nodes that lie to the right (resp.\ left) of node $\nu$. Denote with ${\bf x}_{i,j}$ the value attained by the $j$th feature of the $i$th data point and for $j \in \sets F_{\rm c}$, let $\sets X_j$ collect the possible levels attainable by feature~$j$. Consider the following MIP
\begin{subequations}
\begin{align}
&\text{minimize} \;\;  \ell_{\rm c/\rm r}(\mathcal T, \hat y) + \lambda \ell^{\rm d}_{\rm c/\rm r}(\mathcal T,\hat y) \label{eq:mip_a}\\
&\text{subject to} \;\; p \in \mathcal P, \; \hat y \in \hat{\mathcal Y}(z) \label{eq:mip_b} \\
&  \textstyle q_\nu - \sum_{j \in \sets F_{{\rm q}}} p_{\nu j} {\bf x}_{i,j}= g_{i \nu}^+ - g_{i \nu}^- \quad \qquad \qquad \,\, \forall \nu , \; i 
\label{eq:mip_c} \\
&  \textstyle g_{i \nu}^+ \leq M w_{i \nu}^{\rm q} \; \, \qquad \qquad \qquad \quad \qquad \qquad \qquad \forall \nu, \;  i  
\label{eq:mip_d} \\
&  \textstyle g_{i \nu}^- \leq M (1-w_{i \nu}^{\rm q} )  \qquad \qquad \qquad  \qquad   \quad  \quad \quad \forall \nu, \; i
\label{eq:mip_e}\\ 
&  \textstyle g_{i \nu}^+ + g_{i\nu}^- \geq \epsilon (1-w_{i \nu}^{\rm q}) \quad \qquad \qquad \qquad \;  \quad \quad \forall \nu, \; i 
\label{eq:mip_f}\\
&  \textstyle z_{il} \leq  1-w_{i\nu}^{\rm q}+(1-\sum_{j \in \sets F_{\rm q}} p_{\nu j})  \quad \forall \nu,i,l \in {\sets L}^{\rm r}(\nu) 
\label{eq:mip_g}\\ 
&  \textstyle z_{il} \leq  w_{i\nu}^{\rm q}+ (1-\sum_{j \in \sets F_{\rm q}} p_{\nu j}) \quad \quad \; \forall \nu, \; i, \; l \in {\sets L}^{\rm l}(\nu) 
\label{eq:mip_h} \\
& s_{\nu j k} \leq p_{\nu j} \qquad  \qquad \qquad \qquad \; \forall \nu, \; j \in \sets F_{\rm c}, \; k \in \sets X_j 
\label{eq:mip_i} \\
& w_{i\nu}^{\rm c} = \textstyle \sum_{j\in \sets F_{\rm c}} \sum_{k \in \sets X_j} s_{\nu j k}\I{{\bf x}_{i,j}=k} \quad  \quad \; \; \forall \nu, \;  i
\label{eq:mip_j}\\
&  \textstyle z_{il} \leq  w^{\rm c}_{i \nu} + (1-\sum_{j \in \sets F_{\rm c}} p_{\nu j}) \quad \quad \; \forall \nu, \; i, \; l \in {\sets L}^{\rm l}(\nu)
\label{eq:mip_k}\\
&  \textstyle z_{il} \leq 1 - w^{\rm c}_{i \nu} + (1-\sum_{j \in \sets F_{\rm c}} p_{\nu j}) \quad \forall \nu,i, l \in {\sets L}^{\rm r}(\nu) 
\label{eq:mip_l}\\
& \textstyle \sum_{l \in {\sets L}} z_{il} = 1 \qquad  \qquad \qquad  \qquad \qquad   \qquad \qquad \forall i 
\label{eq:mip_m}
\end{align}
\label{eq:mip}
\end{subequations}
\noindent with variables $p$ and $\hat y$; $q_{\nu}, \; g_{i\nu}^+, \; g_{i\nu}^-\in \reals$; and $z_{il}, \; w^{\rm q}_{i \nu}, \; w^{\rm c}_{i \nu}, \; s_{\nu j k}  \in  \{0,1\}$ for all $i \in \sets N$, $l \in {\sets L}$, $\nu \in \sets V$, $j \in \sets F$, $k\in \sets X_j$, $l \in \mathcal L$.

An interpretation of the variables other than $z$, $p$, and $\hat y$ (which we introduced in Section~\ref{sec:decision-trees-general}) is as follows. The variables $q_{\nu}$,  $g_{i\nu}^+$, $g_{i\nu}^-$, and $w^{\rm q}_{i \nu}$ are used to bound $z_{il}$ based on the branching decisions at each node $\nu$, whenever branching is performed on a quantitative feature at that node. The variable $q_\nu$ corresponds to the cut-off value at node $\nu$. The variables $g_{i\nu}^+$ and $g_{i\nu}^-$ represent the positive and negative parts of $q_\nu - \sum_{j \in \sets F_{{\rm q}}} p_{\nu j} {\bf x}_{i,j}$, respectively. Whenever branching occurs on a quantitative (i.e., continuous or discrete and ordered) feature, the variable $w_{i\nu}^{\rm q}$ will equal 1 if and only if $q_\nu \geq \sum_{j \in \sets F_{{\rm q}}} p_{\nu j} {\bf x}_{i,j}$, in which case the $i$th data point must go left in the branch. The variables $w^{\rm c}_{i \nu}$ and $s_{\nu j k}$ are used to bound $z_{il}$ based on the branching decisions at each node $\nu$, whenever branching is performed on a \emph{categorical} feature at that node. Whenever we branch on categorical feature $j \in \sets F_{\rm c}$ at node $\nu$, the variable $s_{\nu j k}$ equals 1 if and only if the points such that ${\bf x}_{i,j}=k$ must go left in the branch. If we do not branch on feature $j$, then $s_{\nu j k}$ will equal zero. Finally, the variable $w_{i\nu}^{\rm c}$ equals 1 if and only if we branch on a categorical feature at node $\nu$ and data point $i$ must go left at the node.

An interpretation of the constraints is as follows. Constraints~\eqref{eq:mip_b} impose the adequate structure for the decision tree, see Examples~\ref{ex:PWC_policies}-\ref{ex:PWL_policies}. Constraints~\eqref{eq:mip_c}-\eqref{eq:mip_h} are used to bound $z_{il}$ based on the branching decisions at each node $\nu$, whenever branching is performed on a quantitative attribute at that node. Constraints~\eqref{eq:mip_c}-\eqref{eq:mip_f} are used to define $w_{i\nu}^{\rm q}$ to equal 1 if and only if $q_\nu \geq \sum_{j \in \sets F_{{\rm q}}} p_{\nu j} {\bf x}_{i,j}$. Constraint~\eqref{eq:mip_g} stipulates that if we branch on a quantitative attribute at node $\nu$ and the $i$th record goes left at the node (i.e., $w_{i\nu}^{\rm q}=1$), then that record cannot reach any leaf node that lies to the right of the node. Constraint~\eqref{eq:mip_h} is symmetric to~\eqref{eq:mip_g} for the case when the data point goes right at the node. Constraints~\eqref{eq:mip_i}-\eqref{eq:mip_l} are used to bound $z_{il}$ based on the branching decisions at each node $\nu$, whenever branching is performed on a \emph{categorical} attribute at that node. Constraint~\eqref{eq:mip_i} stipulates that if we do not branch on attribute $j$ at node $\nu$, then $s_{\nu j k}=0$. Constraint~\eqref{eq:mip_j} is used to define $w_{i\nu}^{\rm c}$ such that it is equal to 1 if and only if we branch on a particular attribute $j$, the value attained by that attribute in the $i$th record is $k$ and data points with attribute value $k$ are assigned to the left branch of the node. Constraints~\eqref{eq:mip_k} and~\eqref{eq:mip_l} mirror constraints~\eqref{eq:mip_g} and~\eqref{eq:mip_h}, for the case of categorical attributes.

With the loss function taken as the misclassification rate or the mean absolute error and the discrimination loss function taken as one of the indices from Section~\ref{sec:fairness}, Problem~\eqref{eq:mip} is a MIP involving a convex piecewise linear objective and linear constraints. It can be linearized using standard techniques and be written equivalently as an MILP. The number of decision variables (resp.\ constraints) in the problem is $\mathcal O(|\mathcal V| | \mathcal F| \max_j |\mathcal X_j| + |\mathcal N| |\mathcal V| )$ (resp.\ $\mathcal O(|\mathcal V|^2|\mathcal N| + |\mathcal V| |\mathcal F| \max_j |\mathcal X_j|)$, i.e., polynomial in the size of the dataset.

\begin{remark}
Our approach of penalizing unfairness using a regularizer can be applied to existing MIP models for learning optimal trees such as the ones in \cite{verwer2017,Bertsimas2017}. Contrary to these papers which require one-hot encoding of categorical features, our approach yields more interpretable and flexible trees.
\end{remark}

\noindent\textbf{Customizing Interpretability.} An appealing feature of our framework is that it can cater for interpretability requirements. First, we can limit the value of $K$. Second, we can augment our formulation through the addition of linear interpretability constraints. For example, we can conveniently limit the number of times that a particular feature is employed in a test by imposing an upper bound on $\sum_{\nu \in \mathcal V} p_{vj}$. We can also easily limit the number of features employed in branching rules.

\begin{remark}
Preference elicitation techniques can be used to make a suitable choice for $\lambda$ and to learn the relative priorities of decision-makers in terms of the three conflicting objectives of predictive power, fairness, and interpretability.
\end{remark}

\section{Numerical Results}

\textbf{Classification.} We evaluate our approach on 3 datasets: (A) The \texttt{Default} dataset of Taiwanese credit card users \cite{UCI:2017,Yeh:2009} with $|\mathcal N|=30,000$ and $d=23$ features, where we predict whether individuals will default and the protected attribute is gender; (B) The \texttt{Adult} dataset \cite{UCI:2017,Kohavi:1996} with $|\mathcal N|=45,000$, $d=13$, where we predict if an individual earns more than \$50k per year and the protected attribute is race; (C) The \texttt{COMPAS} dataset \cite{Machine_Bias,Corbett:2017} with $|\mathcal N|=10,500$ data points and $d=16$, where we predict if a convicted individual will commit a violent crime and the protected attribute is race. These datasets are standard in the literature on fair ML, so useful for benchmarking. We compare our approach (\texttt{MIP-DT}) to 3 other families: \emph{i)} The MIP approach to classification where $\lambda=0$ (\texttt{CART}); \emph{ii)} the discrimination-aware decision tree approach (\texttt{DADT}) of \cite{Kamiran:2010} with information gain w.r.t.\ the protected attribute (\texttt{IGC+IGS}) and with relabeling algorithm (\texttt{IGC+IGS Relab}); \emph{iii)} The fair \texttt{logistic} regression methods of \cite{Berk:2017} (\texttt{log}, \texttt{log-ind}, and \texttt{log-grp} for regular logistic regression, logistic regression with individual fairness, and group fairness penalty functions, respectively). Finally, we also discuss the performance of an \underline{A}pproximate variant of our approach (\texttt{MIP-DT-A}) in which we assume that individuals that have similar outcomes are similar and replace the distance between features in~\eqref{eq:DTDIc} by the distance between outcomes, as is always done in the literature~\cite{Berk:2017}. As we will see, this approximation results in loss in performance. In all approaches, we conduct a pre-processing step in which we eliminate the protected features from the learning phase. We do not compare to uninterpretable fairness in-processing approaches since we could not find any such approach. 

\begin{figure*}[t!]
        \centering
        \includegraphics[width = \textwidth]{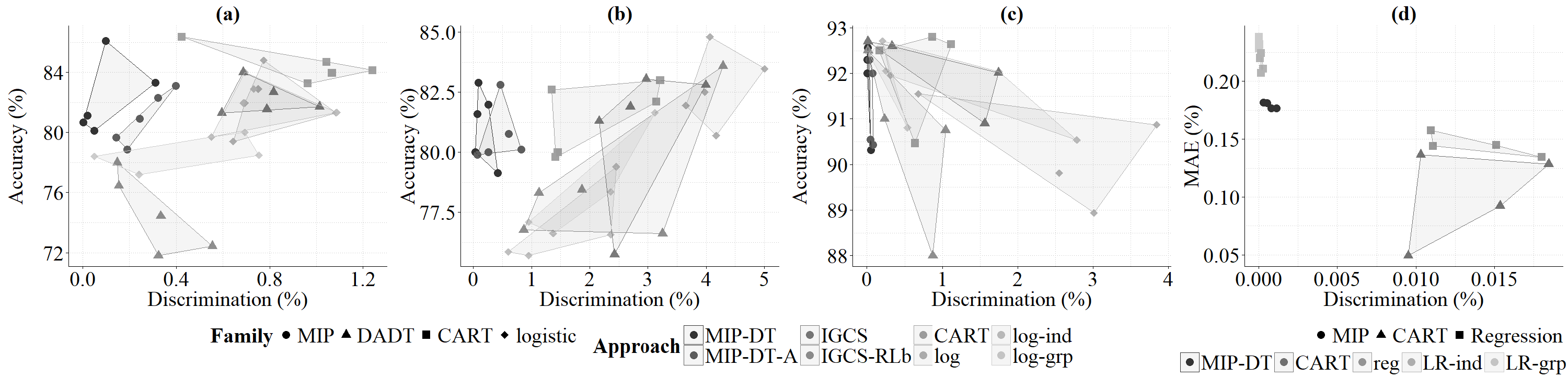}
    \caption{Accuracy-discrimination trade-off of 4 families of approaches
    on 3 classification datasets: (a) \texttt{Default}, (b) \texttt{Adult}, and (c) \texttt{COMPAS}. Each dot represents a different sample from 5-fold cross-validation and each shaded area corresponds to the convex hull of the results associated with each approach in accuracy-discrimination space. Same trade-off of 3 families of approaches 
    on the regression dataset \texttt{Crime} is shown in (d).}
\label{fig:results}
\end{figure*}

\begin{figure*}
\includegraphics[width=0.185\textwidth]{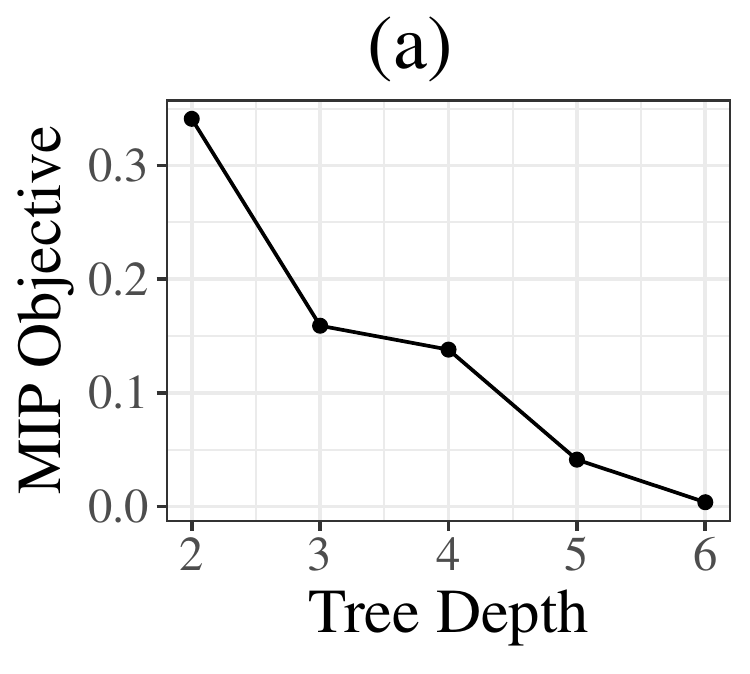}
\includegraphics[width=0.27\textwidth]{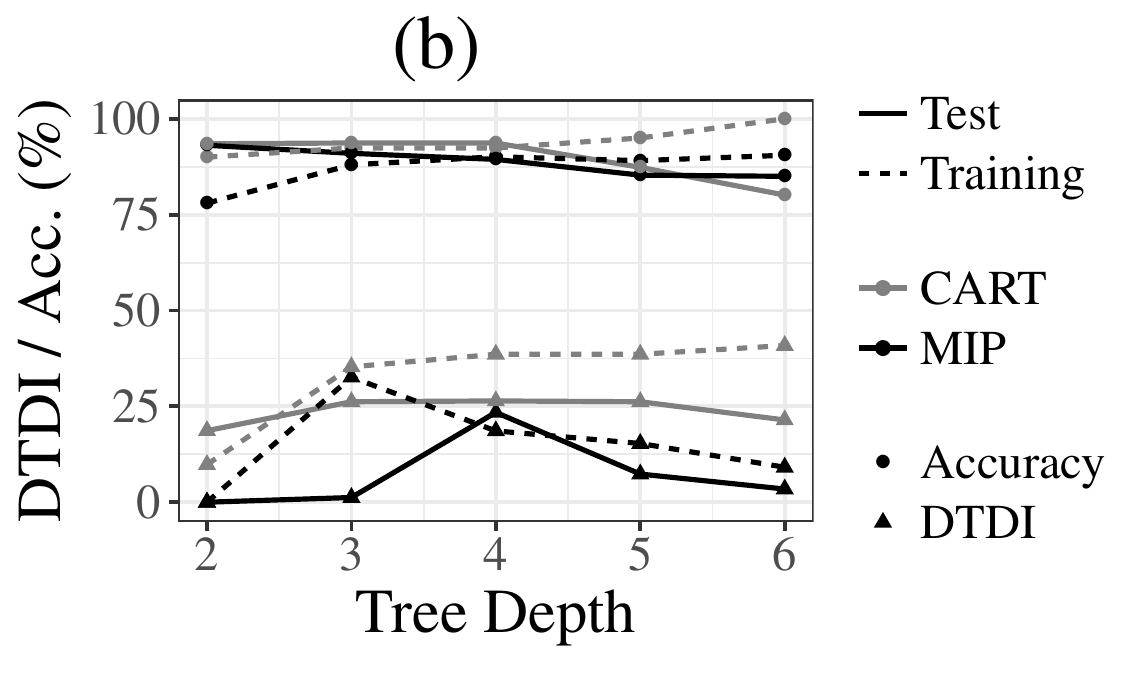}
\includegraphics[width=0.285\textwidth]{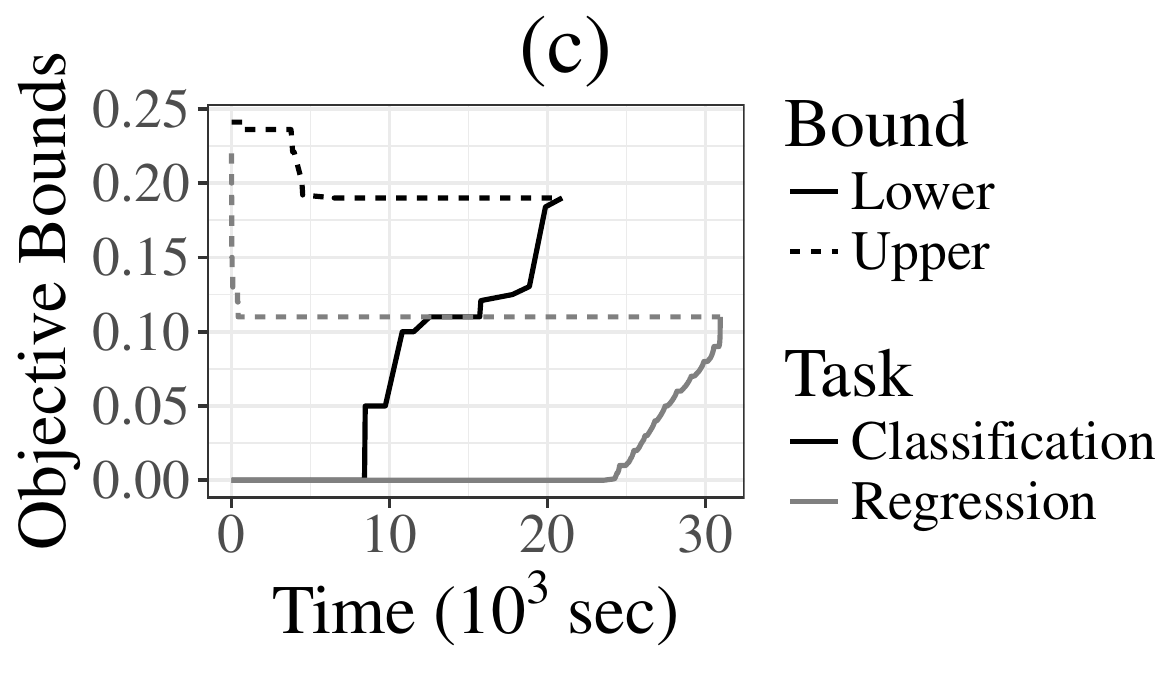}
\includegraphics[width=0.245\textwidth]{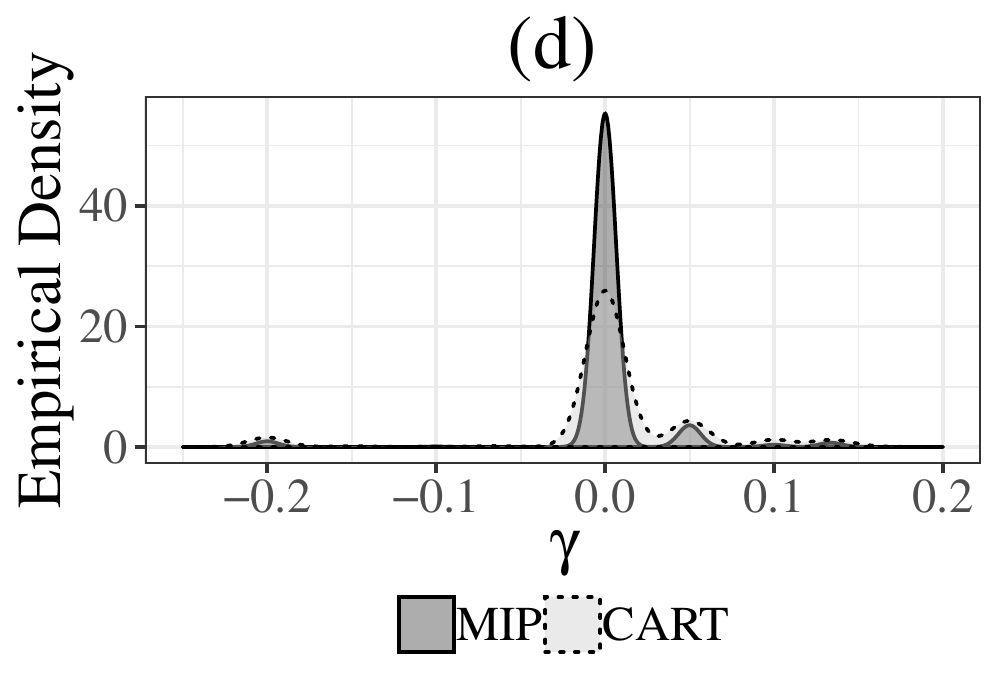}
\caption{From left to right: (a) MIP objective value and (b) Accuracy and fairness in dependence of tree depth; (c) Comparison of upper and lower bound evolution while solving MILP problem; and (d) Empirical distribution of $\gamma(\textbf{x}):= \PR( y |  \uprx, \prx ) - \PR( y | \uprx )$ (see Definition \ref{def:DT_c}) when ${\textbf{x}}$ is valued in the test set in both CART ($\lambda=0$) and MIP. }
\label{fig:intepretability_complexity_fairness}
\end{figure*}

\begin{figure}[t!]
    \centering
        \includegraphics[width=\columnwidth]{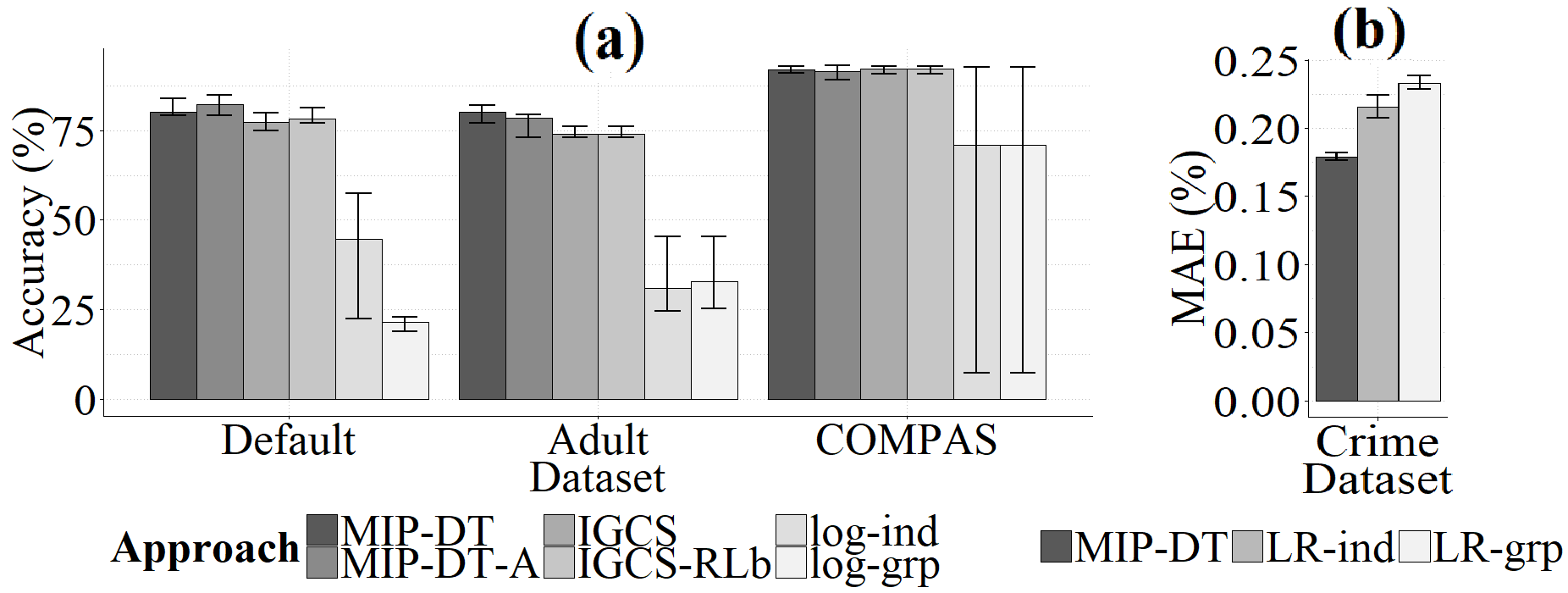} 
        \caption{Accuracy of maximally non-discriminative models in each approach for (a) classification and (b) regression.}
		\label{fig:accuracy_results}
\end{figure}

\noindent\textbf{Regression.} We evaluate our approach on the \texttt{Crime} dataset \cite{UCI:2017,redmond2002data} with $|\mathcal N|=1993$ and $d=128$. We add a binary column called ``race'' which is labeled 1 iff the majority of a community is black and 0 otherwise and we predict violent crime rates using race as the protected attribute. We use the ``repeatedcv'' method in R to select the 11 most important features. We compare our approach (\texttt{MIP-DT} and \texttt{MIP-DT-A}, where \texttt{A} stands for \underline{A}pproximate distance function) to 2 other families: \emph{i)} The MIP regression tree approach where $\lambda=0$ (\texttt{CART}); \emph{ii)} The linear \texttt{regression} methods in \cite{Berk:2017} (marked as \texttt{reg}, \texttt{LR-ind}, and \texttt{LR-grp} for regular linear regression, linear regression with individual fairness, and group fairness penalty functions).

\noindent\textbf{Fairness and Accuracy.} In all our experiments, we use ${\mathsf{DTDI_{\rm c/\rm r}}}$ as the discrimination index. First, we investigate the fairness/accuracy trade-off of all methods by evaluating the performance of the most accurate models with low discrimination. We do $k$-fold cross validation where for classification (regression) $k$ is 5(4). For each (fold, approach) pair, we select the optimal $\lambda$ (call it $\lambda^\star$) in the objective~\eqref{eq:loss_objective} as follows: for each $\lambda$ in $\{0,0.1,0.2,\hdots\}$, we compute the tree on the fold using the given approach and determine the associated discrimination level on the fold; we stop when the discrimination level is $<0.01\%$ and return $\lambda$ as $\lambda^\star$; we then evaluate accuracy (misclassification rate/MAE) and discrimination of the classification/regression tree associated with $\lambda^\star$ on the test set and add this as a point in the corresponding graph in Figure~\ref{fig:results}. For classification (regression), each fold is 1000 to 5000 (400) samples. Figures \ref{fig:results}(a)-(c) (resp.\ (d)) show the fairness-accuracy results for classification (resp.\ regression) datasets. On average, our approach yields results with discrimination closer to zero but also higher accuracy. Accuracy results for the most accurate models with zero discrimination (when available) are shown in Figure~\ref{fig:accuracy_results}. From Figure~\ref{fig:accuracy_results}(a), it can be seen that our approach is more accurate than the fair \texttt{log} approach and has slightly higher accuracy compared to \texttt{DADT}. These improved results come at computational cost: the average solver times for our approach in the 3 classification datasets are\footnote{We modeled the MIP using JuMP in Julia \cite{DunningHuchetteLubin2017} and solved it using Gurobi 7.5.2 on a computer node with 20 CPUs and 64 GB of RAM. We imposed a 5 (10) hour solve time limit for classification (regression).} 18421.43s, 15944.94s and 18161.64s, respectively. The \texttt{log} (resp. \texttt{IGC+IGB}) takes 18.43s, 16.04s, and 7.59s (65.68s, 23.39s, 4.78s).
 Figure \ref{fig:accuracy_results}(b) shows the MAE for each approach for zero discrimination. \texttt{MIP-DT} has far lower error than \texttt{LR-ind/grp}. The average solve time for \texttt{MIP-DT} (resp.\ \texttt{LR-ind/grp}) was 36007 (0.38/0.33) secs. 

\noindent\textbf{Fairness and Interpretability.}  Figures \ref{fig:intepretability_complexity_fairness}(a)-(b) show how the MIP objective and the accuracy and fairness values change in dependence of tree depth (a proxy for interpretability) on a fold from the \texttt{Adult} dataset. Such graphs can help non-technical decision-makers understand the trade-offs between fairness, accuracy, and interpretability. Figure \ref{fig:intepretability_complexity_fairness}(d) shows that the likelihood for individuals (that only differ in their protected characteristics, being otherwise similar) to be treated in the same way is twice as high in \texttt{MIP} than in \texttt{CART} on the same dataset: this is in line with our metric -- in this experiment, DTDI was 0.32\% (0.7\%) for \texttt{MIP} (\texttt{CART}).


\noindent\textbf{Solution Times Discussion.} As seen, our approaches exhibit better performance but higher \emph{training} computational cost. We emphasize that \emph{training} decision-support systems for socially sensitive tasks is usually \emph{not} time sensitive. At the same time, \emph{predicting} the outcome of a new (unseen) sample with our approach, which \emph{is} time-sensitive, is \emph{extremely fast} (in the order of milliseconds). In addition, as seen in Figure~\ref{fig:intepretability_complexity_fairness}(c), a near optimal solution is typically found very rapidly (these are results from a fold from the \texttt{Adult} dataset).

\section*{ Acknowledgments}
The authors gratefully acknowledge support from Schmidt Futures and from the James H. Zumberge Faculty Research and Innovation Fund at the University of Southern California. They thank the 6 anonymous referees whose reviews helped substantially improve the quality of the paper.

\fontsize{9.3pt}{10.3pt} \selectfont

\bibliographystyle{aaai}

\end{document}